\documentclass[Afour,sageh,times]{sagej}
\usepackage{multirow,booktabs}
\usepackage{moreverb,url}
\usepackage{tikz}
\setcitestyle{numbers} 
\usepackage[colorlinks,bookmarksopen,bookmarksnumbered,citecolor=red,urlcolor=red]{hyperref}
\usepackage{graphicx}   
\usepackage{graphics}
\usepackage{epsfig}
\usepackage{amsmath,amssymb,amsthm,amsfonts}
\newtheorem{theorem}{Theorem}
\newcommand\BibTeX{{\rmfamily B\kern-.05em \textsc{i\kern-.025em b}\kern-.08em
T\kern-.1667em\lower.7ex\hbox{E}\kern-.125emX}}

\def\volumeyear{2020}
\begin{document}

\title{Optimizing fire allocation in a NCW-type model}
\author{Nam H. Nguyen, My A. Vu, Anh N. Ta, Dinh V. Bui and Manh D. Hy }
\affiliation{Department of Mathematics, Faculty of Information
Technology, Le Quy Don Technical University, 236 Hoang Quoc Viet Str.,
Bac Tu Liem Dist., Hanoi, Vietnam, 100000.}
\corrauth{Manh Duc Hy, Department of Mathematics, Faculty of
Information Technology, Le Quy Don Technical University, 236 Hoang
Quoc Viet Str., Bac Tu Liem Dist., Hanoi, Vietnam, 100000.}

\email{ducmanhktqs@gmail.com}
\begin{abstract}
In this paper, we introduce a non-linear Lanchester's model of NCW-type and investigate an optimization problem for this model, where only the Red force is supplied by several supply agents. Optimal fire allocation of the Blue force is sought in the form of a piece-wise constant function of time. A "threatening rate" is computed for the Red force and each of its supply agents at the beginning of each stage of the combat. These rates can be used to derive the optimal decision for the Blue force to focus its firepower to the Red force itself or one of its supply agents. This optimal fire allocation is derived and proved by considering an optimization problem of number of Blue force's troops. Numerical experiments are included to demonstrate the theoretical results.
\end{abstract}
 \maketitle
\section{Introduction}
In 1916, Lanchester \citep{Lan} introduced a mathematical model for a battle in the form of a system of differential equations two unknowns of which  are the number of the two involved parties. 
In 1962, Deitchman \citep{Dei} extended Lanchester's model by investigating battle between an army and a guerilla force. This model is called a guerilla warfare model or an asymetric model.
In this model, the fire of guerilla force is supposed to be aimed while of the army is unaimed.
 Later, Schaffer \citep{Sch} and  Schreiber \citep{Sch1} generalized Deitchman's model further by taking into account the intelligence and considered the problem of optimizing the fire allocation of the army.\\
 
Recently, Kaplan, Kress and Szechtman (KKS) \citep{Kap}, \citep{Kre} also considered Lanchester model with intelligence in a scenario of counter-terrorism. In this asymmetric model, intelligence play a decisive role in the outcome of the combat.
In addition, a lot of researchers are interested in optimization problems involving warfare models.	
In 1974, Taylor \citep{Tay} studied several problems of optimizing the fire allocation for some warfare models.
Lin and Mackay \citep{Mac} extended Taylor's results on optimization of fire allocation for Lanchester's model of the form one against many. A common interest of these two studies is optimizing the number of troops. Feichtinger and his colleagues \citep{Fei1} studied an optimization problem for KKS model with objective function being the cost of the battle, intelligence and reinforcement being control variables. Then, the authors investigated a modified asymmetric Lanchester $(n,1)$ model describing a combat between a group of $n$ counter-terrorism forces and a single group of terrorists, \citep{Manh}. In these works above, the role of military supply has not been studied thoroughly. 
	 
The idea of Network Centric Warfare has recently drawn a lot of attention of researchers all over the world. Network Centric Warfare (NCW) emphasizes the role of information and information sharing between entities involved. The idea is originated in 1995 when Admiral William A. Owens introduced the notion "system of systems" in a paper published by INSS, \citep{Owen}. In this work, Owens described the stochastic evolution of a system of reconnaissance, command and control system, together with high-precision weapons, allowing to share battlefield awareness, to estimate the target and to allocate resources. NCW is an approach to the conduct of warfare that derives its power from the effective linking or networking of the warfighting enterprise. It is characterized by the ability of geographically dispersed forces (consisting of entities) to create a high level of shared battlespace awareness that can be exploited via self-synchronization and other network-centric operations to achieve commanders' intents \citep{Alb}. For a more thorough exposition of NCW, we refer the readers to \citep{Smi1}, \citep{Tun}, \citep{Tun1}.

By historical facts, it is undeniable that supply forces also play a vital role in the outcome of the battle. On the other hand, the optimal decision making problem in military field is interesting itself and fire allocating problem is one of the most common problems. In this regard, Donghyun Kim \citep{Kim} set up and investigated a non-linear Lanchester-type model where one of the parties is supported by a network consisting of all kinds of supply such as intelligence, ammunition, medical, etc. 
However, there are various types of supply an army should be provided during a combat. Each of them plays a different role in the combat and their affects are different, too. Nevertheless, all these kinds of supply to a party can be eliminated by the firepower of the other one for the sake of victory. 

In this work, exploiting the idea of NCW and bearing in mind the role of military logistics, we set up a model where a Blue force $B$ is fighting against a Red force $R$ supported by $n$ supply agents $A^i\, (i=1,\ldots , n)$. These supply agents have different affects on the attrition rate of $R$. By assuming that the affects are all described by linear functions and  investigating the resulting model, we manage to derive an optimized strategy for the Blue force $B$ to keep its status at its best.
		 The rest of the paper is organized as follows. Next section is devoted to present model setting and to investigate the optimization problem for this model. Numerical experiments are presented in the last section to illustrate the main results.

\section{Main Results}
\subsection{The Model}
Let us consider a combat between a Blue force and a Red force and assume that the Red force is supplied by $n$ different supply agents. 
We use the following notations:\\ 
\begin{itemize}
\item$B$: Blue force
\item$R$: Red force (Attack)
\item$A^i\, (i=1,...,n)$: Red force's supply agents
\item$r_R$: an attrition rate of $B$ to $R$
\item $r_i$: an attrition rate of $B$ to $A^i$
\item $f^i_\alpha(A^i)$: an attrition function of $A_i$ complementing $R$ to $B$.
\item $P=(p_0,p_1,...,p_n)$: the fire allocating proportion of $B$ to $R$ and $A^i:\, i=1,\ldots , n$, respectively.
\item $\alpha^{A^i}_{c}:$ the fully connected attrition rate of $R$ with $A^i$ 
\item $\alpha^{A^i}_{d}:$ the fully disconnected attrition rate of $R$ with $A^i$ $(\alpha^{A^i}_{d}\leq\alpha^{A^i}_{c})$.
\end{itemize}
 
 For the sake of simplicity, we study a model where $B$ are fighting against $R$ and $R$ have two supply agents ${A^1}$ and $A^2$ with two complementing fire attrition functions ${f^1_\alpha }\left( A^1 \right),\, {f^2_\alpha }\left( {A^2} \right)$, respectively.  
\begin{figure}
 \includegraphics[scale=1]{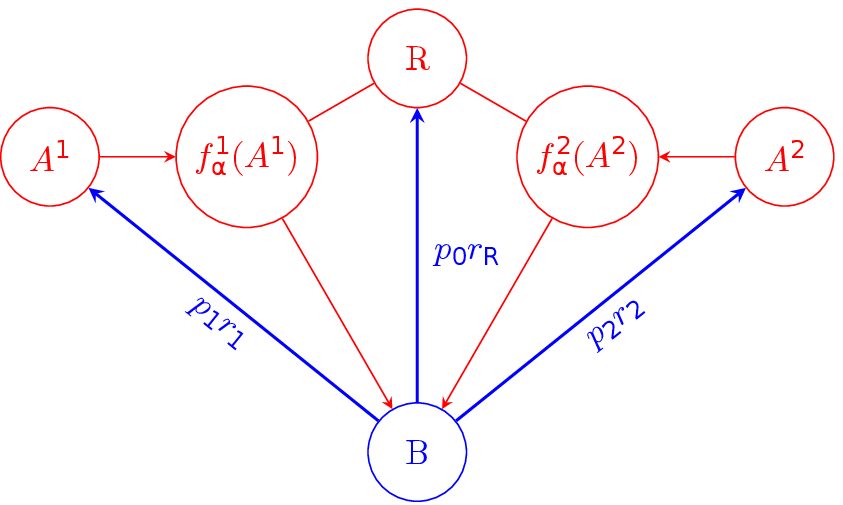}
 \caption{Diagram for the model $\left( {B \texttt{ vs } (R,A^1,{A^2})} \right)$.}
 \label{hinh1}
 \end{figure}
We denote this model as $\left( {B \texttt{ vs } (R,A^1,{A^2})} \right)$. The diagram for this model is presented in Figure \ref{hinh1}. 
Let us consider the problem of finding the optimal fire allocation of $B$ such that at any time $t$, the remaining troops of $B$ is maximized. We seek for the optimal fire allocating proportion of $B$ in the form of a piece-wise constant function. This choice is realistic since it is absurd to alter the fire allocation constantly, especially during a certain stage of the battle. For this purpose, we assume that $P = \left( p_0 ,p _1,p _2 \right)$ where $p_i:\, i=0,1,2$ are piece-wise constant functions such that $p_0 ,p _1,p _2 \in \left[0,1\right]:\,p_0  + p _1 + p _2 = 1$.  \\
The complementing attrition functions are assumed to be linear ones of the form:
\begin{equation*}
f^i_\alpha(A^i)  = \alpha _d^{A^i} + (\alpha _c^{A^i} - \alpha _d ^{A^i})\frac{A^i}{A^i_0}\, ( i=1,2).
\end{equation*}
where  $A^i_0$ number of $A^i$'s troops at the beginning. Let us observe that, at the beginning, when $A^i=A^i_0$, $R$ and $A^i$ has a full connection and $f^i_\alpha(A^i)$ attains its maximal value $\alpha _c^{A^i}$. When $A^i$ is totally eliminated by $B$, $A^i=0$, the connection between $R$ and $A^i$ is terminated and $f^i_\alpha(A^i)$ becomes $\alpha _d^{A^i}$.
 The numbers of troops of all the parties involved in the battle are governed by the following system of differential equations:
 \begin{equation}\label{model}
\left \{
 \begin{aligned}
 \frac{dB}{dt} &= -\left(f^1_\alpha(A^1) + f^2_\alpha({A^2})\right)R\\
 \frac{dR}{dt}&=-p_0r_R B\\
 \frac{dA^1}{dt}&=-p _1 r_1B\\
  \frac{d{A^2}}{dt}&=-p _2 r_2B.
 \end{aligned}
 \right.
 \end{equation}
 It is apparent that supply agents $A^1,A^2$ create their impacts on the outcome of the battle by influencing the attrition rate of $R$ to $B$. When their number of troops is eliminated, their impacts are stopped accordingly. 
 \subsection{Optimal fire allocation of Blue force}
For the model \eqref{model}, we consider the problem of maximizing the Blue force's number of troops at any time.
 Let us compute the following:
 \begin{equation}
\begin{aligned}
b_0&=&(\alpha_c^{A^1}+\alpha_c^{A^2})r_R, \\
b_1&=&\frac{r_1(\alpha_c^{A^1}-\alpha_d^{A^1})R_0}{A^1_0},\\
b_2&=&\frac{r_2(\alpha_c^{A^2}-\alpha_d^{A^2})R_0}{{A^2}_0}.
\end{aligned}\label{b012}
 \end{equation}
These numbers represent the "threatening rates" which the Red force and its supply agents expose to the Blue force. The optimal fire allocation of $B$ is pointed out in the following theorem.
\begin{theorem}\label{theo1}
If $p_0, p_1, p_2$ are piece-wise constant functions and $p_0,p_1,p_2 \in [0,1]: p_0+p_1+p_2 =1$ then the optimal fire allocation of $B$ is 
\begin{equation}
P ^* = 
\begin{cases} 
(1,0,0) \textrm{ if } b_0\geq b_1 \textrm{ and } b_0\geq b_2,\\
(0,1,0) \textrm{ if } b_1\geq b_2 \textrm{ and } b_1\geq b_0,\\
(0,0,1) \textrm{ if } b_2\geq b_1 \textrm{ and } b_2\geq b_0.
\end{cases} 
\end{equation}
\end{theorem}
\begin{proof}
Let $X(t) = \int _0^t B(s) ds$. It follows that $X' (t) = B(t)$ and 
\begin{equation}\label{d2X}
\begin{aligned}
X''(t) &=B'(t) \\
& = -\left( f_\alpha (A^1) + g_\alpha ({A^2})\right)R.
\end{aligned}
\end{equation}
We also have 
$$\int _0^t dR=-\int_0^t r _R B(s) ds \Rightarrow R(t) -R(0) = -p_0 r _R X(t).$$ This leads to 
\begin{equation} \label{Rt}
R(t) = -p_0r _R X(t) + R_0.
\end{equation}
By similar arguments, we get 
\begin{eqnarray}
A^1(t) = -p_1 r _1X(t) + A^1_0, \label{A^1t}\\
{A^2}(t) = -p _2 r _2 X(t) + A^2_0.\label{A^2t}
\end{eqnarray}
Substituting \eqref{Rt}, \eqref{A^1t} and \eqref{A^2t} into \eqref{d2X} we obtain
\begin{equation}\label{d2XbyC}
X''(t) = -C_1 X^2(t) +C_2 X(t) - C_3
\end{equation}
where 
\begin{equation*}
\begin{aligned}
C_1 =& \frac{p_0 p_1 r _R r _1 (\alpha_c^{A^1}-\alpha_d^{A^1})}{A^1_0} + \\
&\frac{p_0 p_2 r _R r_2 (\alpha_c^{A^2}-\alpha_d^{A^2})}{A^2_0}, \\
C_2=&\frac{p_1 r_1 (\alpha_c^{A^1}-\alpha_d^{A^1})R_0 + p_0 r _R \alpha_c^{A^1}A^1_0}{A^1_0} +\\
&\frac{p_2 r_2 (\alpha_c^{A^2}-\alpha_d^{A^2})R_0 + p_0 r _R \alpha_c^{A^2} A^2_0}{A^2_0},\\
C_3=&(\alpha_c^{A^1} + \alpha_c^{A^2}) R_0.
\end{aligned}
\end{equation*}
Multiplying both sides of \eqref{d2XbyC} by $dX'(t)$ and integrating, one gets 
\begin{equation*}
\begin{aligned}
X'(t) & = B(t) \\
& = \sqrt{-\frac{2}{3}C_1 X^3(t) + C_2X^2(t) -2C_3X(t) +C_4},
\end{aligned}
\end{equation*}
where $C_4$ is an integral constant. Since $C_3$ is not changing in time and $C_1, C_2$ are non-negative, in order to maximize $B(t)$, we will seek for conditions for which $C_1$ is minimal and $C_2$ is maximal simultaneously. Thus, we consider the multi-objective optimization problem 
\begin{equation}\label{moo}
\begin{cases}
\min \limits _{P\in{\mathcal{P}}} C_1\\
\max\limits _{P\in{\mathcal{P}}} C_2,
\end{cases}
\end{equation}
where $\mathcal{P}=\left\lbrace \left( p_0,p_1,p_2\right): 0\leq p_0,p_1,p_2\leq 1, p_0+p_1+p_2=1\right\rbrace $.\\
Let us denote 
$$
a_1 = \frac{r _R r _{A^1}(\alpha _c^{A^1}-\alpha_c^{A^1)}}{A^1_0}, a_2 = \frac{r _R r _{A^2} (\alpha _c^{A^2}-\alpha_c^{A^2})}{A^2_0}.
$$
The problem \eqref{moo} now takes the form
\begin{equation}\label{moo_xyz}
\begin{aligned}
&\begin{cases}
\min (a_1 xy + a_2 xz) \\
\min (-b_0x - b_1y -b_2z),
\end{cases} \\
&\text{ s.t. }
\begin{cases}
0 \leq x,y,z \leq 1\\
 x+y+z=1.
\end{cases}
\end{aligned}
\end{equation}
In order to solve the problem \eqref{moo_xyz}, we use the scalarization method. Thus, for each $\gamma \in [0.1]$ we define the function 
$$
F_\gamma (x,y,z)=\gamma (a_1xy + a_2xz)-(1-\gamma)(b_0x+b_1y+b_2z)
$$
and consider the following problem
\begin{equation}\label{e12}
\begin{aligned}
& \min F_\gamma (x,y,z)\\
&\text{ s.t. }
\begin{cases}
0 \leq x,y,z \leq 1\\
 x+y+z=1.
\end{cases}
\end{aligned}
\end{equation}
By substituting $x=1-y-z$ into \eqref{e12} we obtain the following problem:
\begin{equation}
\begin{aligned}
&\min F_\gamma (1-y-z,y,z)\\
& \text{ s.t } 
\begin{cases}
0 \leq y,z \leq 1\\
y+z\leq 1,
\end{cases}
\end{aligned}
\end{equation}
where
\begin{equation*}
\begin{aligned}
F_\gamma \left( {1 - y - z,y,z} \right) =  \gamma \left( {1 - y - z} \right)\left( {{a_1}y + {a_2}z} \right) -\\ \left( {1 - \gamma } \right)\left( {b_0 + \left( {{b_1} - b_0} \right)y + \left( {{b_2} - b_0} \right)z} \right).
\end{aligned}
\end{equation*}
We consider the following distinct cases:
\begin{itemize}
\item[1.] $b_1\geq b_2\geq b_0$. Since $\gamma(1-y-z)(a_1y+a_2z)\geq 0$, one gets 
\begin{equation*}
\begin{aligned}
\min F_\gamma \geq & -(1-\gamma)(b_0+(b_1-b_0)y+(b_2-b_0)z)\\
\geq & -(1-\gamma)(b_0+(b_1-b_0)(y+z))\\
\geq & -(1-\gamma)b_1 = F_\gamma(0,1,0).\\
\end{aligned}
\end{equation*}
\item[2.]  $b_1\geq b_0\geq b_2$. The problem becomes 
\begin{multline*}
\min \{\gamma (1-y-z)(a_1y+a_2z) \\
+(1-\gamma)(b_0-b_2)z \\
-(1-\gamma)(b_0+(b_1-b_0)y) \}.
\end{multline*}
It easily follows that $$\min F_\gamma \geq -(1-\gamma) b_1 = F_\gamma(0,1,0).$$
\item[3.]  $b_2\geq b_1\geq b_0$. Simple calculations yield that 
\begin{equation*}
\begin{aligned}
\min F_\gamma \geq & -(1-\gamma)(b_0+(b_2-b_0)(y+z))\\
\geq & -(1-\gamma)b_2 = F_\gamma(0,0,1).
\end{aligned}
\end{equation*}
\item[4.] $b_2\geq b_0\geq b_1$. The problem becomes
\begin{multline*}
\min \{\gamma (1-y-z)(a_1y+a_2z) \\
+(1-\gamma)(b_0-b_1)y \\
-(1-\gamma)(b_0+(b_2-b_0)z) \}.
\end{multline*}
It leads to $$\min F_\gamma \geq -(1-\gamma) b_2 = F_\gamma(0,0,1).$$

\item[5.]  $b_0\geq b_1,\; b_0\geq b_2$. The problem now turns out to be 
\begin{multline*}
\min \{\gamma (1-y-z)(a_1y+a_2z) \\
+(1-\gamma)((b_0-b_1)y+(b_0-b_2)z) \\
-(1-\gamma)b_0 \}.
\end{multline*}
\end{itemize}
One obtains, in this case, that 
$$
\min F_\gamma \geq -(1-\gamma)b_0 = F_\gamma(1,0,0).
$$
The proof is now complete.
\end{proof}
Basically, the battle has three stages. In the first stage, the Blue force focuses its firepower to one of the three entities of the Red party: $R,A_1,A_2$. Once the targeted entity is eliminated, in the second stage, the Blue force concentrates its troops to fight one of the two remaining entities. In the last stage, the Blue force focuses its power to the last remaining entity. However, the battle may not necessarily have all three stages, once the Red force is eliminated after some stage the battle comes to an end since the Blue force is no longer attrited. In our proof for Theorem \ref{theo1}, by comparing the three "threatening rates", we establish the optimal fire allocation of Blue force for the first stage. If Red force remains untouched after the first stage, there is only one supply agent remaining, we may apply Kim's results \citep{Kim}. In case 1 of our five cases, in order to make the remaining troops of Blue force maximal, it should strategically focus all its firepower to $A^1$. 
\begin{figure}
\includegraphics[scale=1]{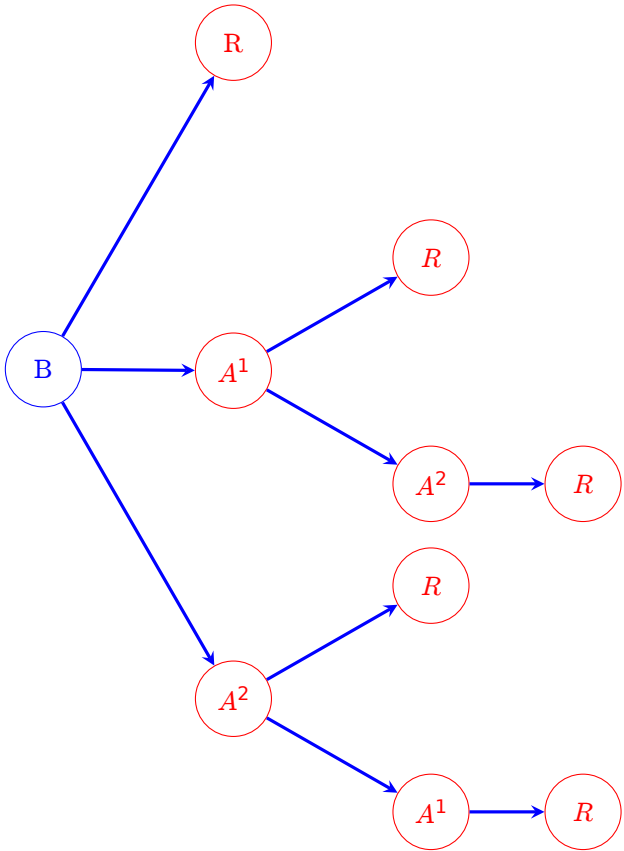}
\caption{Five cases and their corresponding processes in Theorem \ref{theo1}.}\label{graph}
\end{figure}
When $A^1$ is excluded, the "threatening rates" are now $b_0=\left(\alpha _c^{A^2} + \alpha _d^{A^1}\right) r_R,\; 
b_2 = \frac{r _2(\alpha_c^{A^2} -\alpha_d^{A^2} )R_0}{A_0^2}$. By  Kim's results, if $b_0 < b_2$, Blue force should concentrate its troops to fight $A^2$. Once $A^2$ is out of the picture, there comes the third stage where Blue force fights the Red force with no supply agents. Otherwise, $b_0 > b_2$, second stage of the battle commences with Blue force focusing its power to fight Red force and the battle will come to an end after this stage. For the remaining four cases, the strategy of Blue force is analyzed similarly. All five cases of the battle and their corresponding processes are described in the graph of Figure \ref{graph}. 
\begin{figure}
\includegraphics[scale=1]{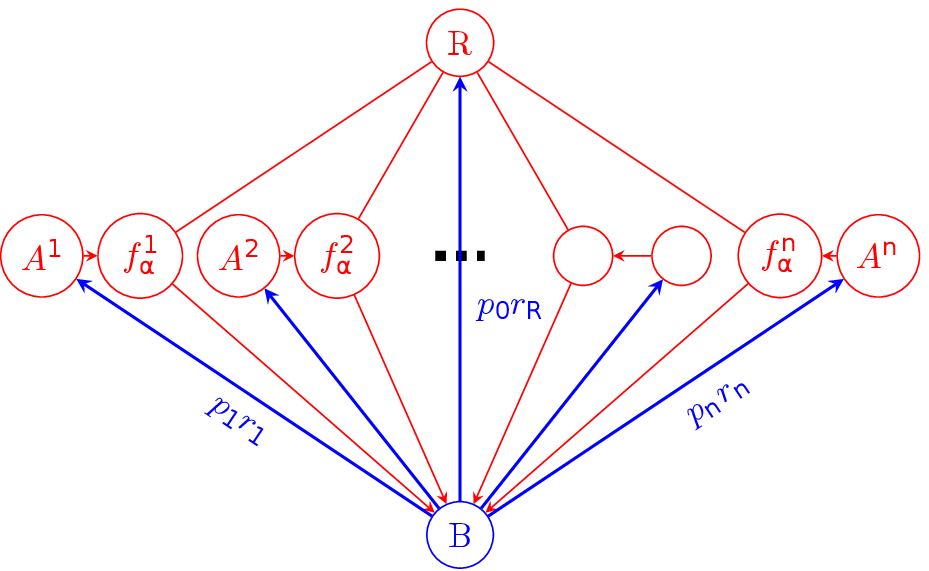}
\caption{The generalized model $\left( B \text{ vs } (R,A^1,A^2,\ldots , A^n )\right).$}\label{n1}
\end{figure}
\\
Now, we generalize this model where the Red force supplied by $n$ supply agents fights the Blue force with an attrition rate $f^1_\alpha (A^1) + f^2_\alpha (A^2) + \cdots + f^n_\alpha (A^n)$. We denote the generalized model as $\left( B \text{ vs } (R,A^1,A^2,\ldots , A^n )\right)$.
Let us assume, as before, that 
\begin{equation} \label{falphas}
f^i_\alpha (A^i) = \alpha_d^{A^i} + (\alpha_c^{A^i}-\alpha_d^{A^i}) \frac{A^i}{A^i_0}, i=1,2,\ldots , n.
\end{equation}
The model is described by the following system of differential equations: 
\begin{equation}
\begin{cases}
\frac{dB}{dt} =  - \sum \limits_{i = 1}^n \left[ \alpha _d^{A^i} + \left( \alpha _c^{A^i} - \alpha _d^{A^i} \right) \frac{A^i}{A_0^i} \right] R\\
\frac{dR}{dt} =  - p_0  r _R B\\
\frac{d A^i}{dt} =  - p _i r _{i} B\,(i = 1, 2,\ldots , n).
\end{cases}
\end{equation}
Let 
\begin{equation} \label{ai}
a_i=\frac{r_R r_{i}(\alpha_c^{A^i}-\alpha_d^{A^i})}{A_0^i}\,(i=1,2,\ldots , n)
\end{equation}
and 
\begin{equation}\label{bi}
\begin{aligned}
b_0 &= \left( \alpha _c^{A^1} + ... + \alpha _c^{A^n} \right){r _R},\\
b_i &= \frac{r _{i} \left( \alpha _c^{A^i} - \alpha _d^{A^i}\right) R_0}{A_0^1} \,(i=1,2,\ldots , n).
\end{aligned}
\end{equation}
For the generalized model, we have the following result.
\begin{theorem}
Suppose that the optimal fire allocation $P^*$ is sought in the set $\mathcal{P} =\{ P=(p_0,\,p_1,\,p_2,\ldots ,p_n): p_i \text{ is piece-wise constant function },\, p_i \in [0,1] \; \forall i; \sum \limits_{i=0}^n p_i =1\} $. Then the optimal fire allocation of $B$ is\\
${P^*} = \left( {0,...,0,\underbrace 1_{i^{th}},0,...,0} \right)\,where\,\,i = \mathop {\arg \max }\limits_{j = 0,1,...,n} \left\{ {{b_j}} \right\}.$
\end{theorem}
\begin{proof}
Let  $X\left( t \right) = \int\limits_0^t {B\left( s \right)ds \Rightarrow } X'\left( t \right) = B\left( t \right)$\\  
By similar calculations, one gets:\\
\begin{equation*}
\begin{aligned} 
X'\left( t \right) &= B\left( t \right)\\
& = \sqrt { - \frac{2}{3}{C_1}{X^3}\left( t \right) + {C_2}{X^2}\left( t \right) - 2{C_3}X\left( t \right) + {C_4}} .
\end{aligned}
\end{equation*}
where 
\begin{equation*}
\begin{aligned}
C_1 &= \sum\limits_{i = 1}^n \frac{p_0 p_i r _R r _i \left( \alpha _c^{A^i} - \alpha _d^{A^i} \right)}{A_0^i}, \\
C_2 &= \sum\limits_{i = 1}^n \frac{p_ir _i\left( \alpha _c^{A^i} - \alpha _d^{A^i} \right)R_0 + p_0r _R\alpha _c^{A^i} A_0^i}{A_0^i}, \\
C_3 &= R_0\sum\limits_{i = 1}^n {\alpha _c^{A^i}}.
\end{aligned}
\end{equation*}
Our optimal problem now becomes :  $\left\{ \begin{array}{l}
\mathop {\min }\limits_{P\in \mathcal{P}}  {C_1}\\
\mathop {\max }\limits_{P\in \mathcal{P}}  {C_2}.
\end{array} \right.$\\
Using notations in \eqref{ai} and \eqref{bi}, the problem becomes \\
\begin{equation}
\begin{aligned}
&\begin{cases}
\min \sum \limits_{i=1}^n a_ix_0x_i\\
\min \sum \limits_{i=0}^n -b_i x_i
\end{cases}
& s.t.
\begin{cases}
0 \le {x_i} \le 1,\\
\sum \limits_{i=0}^n x_i = 1
\end{cases}
\end{aligned}
\end{equation}

Using the scalarization method and for $\gamma \in [0,1]$, setting \\
\begin{equation*}
F_\gamma \left( {x_0,{x_1},\ldots,{x_n}} \right) =  \gamma \sum\limits_{i=1}^n a_ix_0x_i  
 - \left( {1 - \gamma } \right)\sum\limits_{i=0}^n b_ix_i, 
\end{equation*}
we now get the problem:\\
\begin{equation*}
\begin{aligned}
&\min F_\gamma \left( {x_0,{x_1},...,{x_n}} \right)\\
&s.t.
\begin{cases}
0 \le {x_i} \leq 1 \,(i=0,1,\ldots ,n),\\
\sum\limits_{i = 0}^n {{x_i}}  = 1.
\end{cases}
\end{aligned}
\end{equation*}
By considering cases like in Theorem \ref{theo1} and using similar arguments, we complete the proof.
\end{proof}
\section{ Numerical illustrations}
In this section, we will present numerical experiments for three cases of Theorem \ref{theo1}, where the Red force is attacked in the first, the second and the third stage, respectively. 
\subsection{Experiment 1: The Red force is attacked in the first stage}
In this experiment, we consider a battle between Blue force and Red force supported by two supply agents $A^1,A^2$ with initial number of troops are 
$$R_0=120, A_0^1=30, A_0^2=20, B_0=160.$$
The attrition rates of $R$ with $A^1,A^2$ and those of $B$ to $R$ and $A^1,A^2$ are given in the following table
\begin{center}
\begin{tabular}{|c|c|c|c|}
\hline
 \multirow{2}{*}{Case 1}&$(\alpha_d^{A^1},\alpha_c^{A^1})$ & $(\alpha_d^{A^2},\alpha_c^{A^2})$ & $(r_R,r_1,r_2)$  \\
 \cline{2-4}
& $(0.15, 0.4)$ & $(0.1, 0.3)$ & $(0.5, 0.3, 0.2)$ \\
 \hline
\end{tabular}
\end{center}
We computed the "threatening rates" by \eqref{b012} and obtained $$b_0=0.35, b_1=0.3, b_2=0.24.$$ 
Therefore, the optimal fire allocation for Blue force in the first stage is $P^*=(1, \;0, \; 0)$, i.e, Blue force concentrates its firepower to the Red force, not its supply agents. Besides the optimal fire allocation, we also consider two other fire allocations $P_1=(0.7,\; 0.2,\; 0.1)$ and $P_2=\{(0,\; 1,\; 0) \rightarrow (1,\; 0,\; 0)\}$. The allocation $P_2$ should be interpreted as follows: in the first stage, Blue force focuses its firepower to fight supply agent $A^1$, after $A^1$ is eliminated, it attacks the Red force with its full power.
\begin{figure}
\includegraphics[scale=0.25]{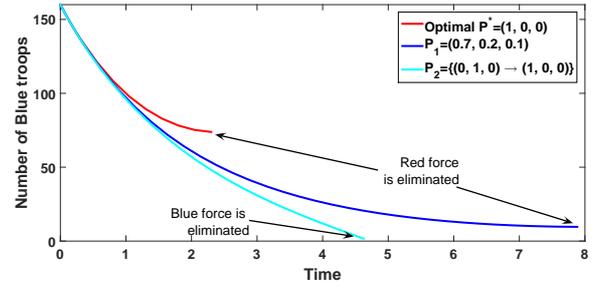}
\caption{Number of Blue troops \textit{vs} time in Experiment 1}.
\label{exp1}
\end{figure}
The outcome of these calculations are presented in Figure \ref{exp1}
\subsection{Experiment 2: The Red force is attacked after one of its supply agents is excluded}
We slightly modify the parameters used in Experiment 1, thus make 
\begin{center}
\begin{tabular}{|c|c|c|c|}
\hline
 \multirow{2}{*}{Case 2}&$(\alpha_d^{A^1},\alpha_c^{A^1})$ & $(\alpha_d^{A^2},\alpha_c^{A^2})$ & $(r_R,r_1,r_2)$  \\
  \cline{2-4}
& $(0.15, 0.4)$ & $(0.1, 0.3)$ & $(0.5, 0.15, 0.4)$ \\
 \hline
\end{tabular}
\end{center}
The initial conditions are kept as in Experiment 1 and the "threatening rates" now are 
$$
b_0=0.35,\, b_1=0.15,\, b_2=0.48.
$$
According to these rates, the optimal fire allocation for the first stage is now $P^*=[0,\;0,\; 1]$. When the first stage finishes at $t_1=0.7536$, the troops of the Red force and agent $A^1$ remain unchanged, while the Blue troops is $B(t_1)=106.3$. The fully-disconnected attrition rate of $R$ with $A^2$ is added to the attrition rates of $R$ with $A^1$ and make these rates now 
$$ (\tilde{\alpha}_d^{A^1}, \tilde{\alpha}_c^{A^1}) = (0.25,\;0.5).$$  
Threatening rate of Red force is now $\tilde{b_0}=\tilde{\alpha}_c^{A^1} r_R =0.25$, while threatening rate of $A^1$ remains the same as $b_1=0.15$.
It follows that the next target of Blue force is the Red force. After finishing the Red troops, Blue force puts an end to the battle. Hence, the optimal fire allocation is 
$$P^* = \{(0,\; 0,\; 1) \rightarrow (1,\; 0,\; 0)\}.$$
To make a comparison, two fire allocations are selected: 
$$P_1=(1,\; 0,\; 0) \text{ and } P_2 =(0,\; 1,\; 0).$$
\begin{figure}
\includegraphics[scale=0.25]{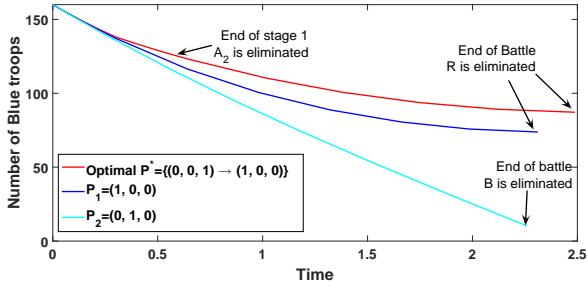}
\caption{Number of Blue troops \textit{vs} time in Experiment 2}.
\label{exp2}
\end{figure}
Figure \ref{exp2} depicts the conclusion of the three choices. Among three choices of fire allocations, $P^*$ and $P_1$ result in victory for Blue force. The battle ends with the failure of Blue force with fire allocation $P_2$.
\subsection{Experiment 3: Blue force attacks two supply agents of Red force first}
In this experiment, we consider equation \eqref{model} with 
\begin{center}
\begin{tabular}{|c|c|c|c|}
\hline
 \multirow{2}{*}{Case 3}&$(\alpha_d^{A^1},\alpha_c^{A^1})$ & $(\alpha_d^{A^2},\alpha_c^{A^2})$ & $(r_R,r_1,r_2)$  \\
 \cline{2-4}
& $(0.15, 0.4)$ & $(0.1, 0.3)$ & $(0.5, 0.3, 0.4)$ \\
 \hline
\end{tabular}
\end{center}
and initial conditions
$$
R_0=120, \,A_0^1=30, \,A_0^2=20,\, B_0=200.
$$
Threatening rates of the Red party are 
$$
b_0=0.35,\; b_1=0.3,\; b_2=0.48,
$$
and thus, for the first stage, Blue force focuses its firepower to $A^2$. Once the first stage finishes, the remaining troop of Blue force is $182$. The fully-disconnected attrition rate of $R$ with $A^2$ is added to the attrition rates of $R$ with $A^1$ and make these rates now 
$$ (\tilde{\alpha}_d^{A^1}, \tilde{\alpha}_c^{A^1}) = (0.25,\;0.5).$$  
\begin{figure}
\includegraphics[scale=0.25]{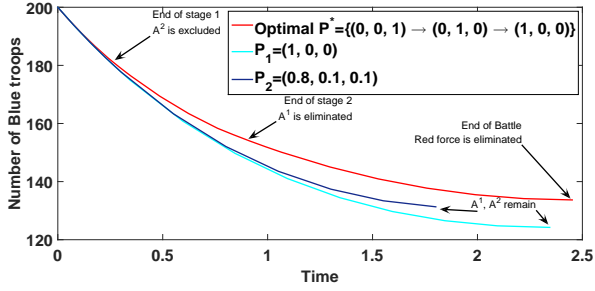}
\caption{Number of Blue troops \textit{vs} time in Experiment 3}.
\label{exp3}
\end{figure}
Threatening rate of Red force is now $\tilde{b_0}=\tilde{\alpha}_c^{A^1} r_R =0.25$, while threatening rate of $A^1$ remains the same as $b_1=0.3$. Therefore, in the second stage, Blue force will concentrate on fighting agent $A^2$. And for the third stage, Blue force will conclude this battle with a fight against $R$ with its full firepower. These three stages are depicted in Figure \ref{exp3} (the red curve). 
\section{Conclusion}
In this work, we have introduced a novel model for battle with supplies. Computing the "threatening rates" of the Red party's entities, we managed to show the optimal fire allocation of the Blue party. 
Generally, at the beginning of any stage, Blue force will concentrate its firepower to the entity which possesses the largest threatening rates; at the end of the stage, threatening rates are recalculated. This process is repeated until the battle ends. 
These results have generalized some known results in this field. 
\bibliography{mohinhtoan}

\begin{thebibliography}{10}

\bibitem{Lan}
Lanchester FW.
\newblock Aircraft in Warfare: The Dawn of the Fourth Arm.
\newblock {C}onstable, London; 1916.

\bibitem{Dei}
Deitchman SJ.
\newblock A {L}anchester model of guerilla warfare.
\newblock {O}perations {R}esearch. 1962.

\bibitem{Sch}
Schaffer MB.
\newblock Lanchester models of guerrilla engagements.
\newblock {O}perations {R}esearch. 1968.

\bibitem{Sch1}
Schreiber TS.
\newblock Letter to the Editor—Note on the Combat Value of Intelligence and
  Command Control Systems.
\newblock {O}perations {R}esearch. 1964.

\bibitem{Kap}
Kaplan EH, Mintz A, Shaul M, Samban C.
\newblock What Happened to Suicide Bombings in Israel? Insights from a Terror
  Stock Model.
\newblock Studies in Conflict and Terrorism. 2008.

\bibitem{Kre}
Kress M, Szechtman R.
\newblock Why defeating insurgencies is hard: the effect of intelligence in
  counter insurgency operations - a best case scenario.
\newblock {O}perations {R}esearch. 2009.

\bibitem{Tay}
Taylor JG.
\newblock Lanchester-type models of warfare and optimal control.
\newblock {Naval {R}esearch {L}ogistics {Q}uarterly}. 1974.

\bibitem{Mac}
Lin KY, MacKay NJ.
\newblock The optimal policy for the one-against-many heterogeneous Lanchester
  model.
\newblock {O}perations {R}esearch Letters. 2014.

\bibitem{Fei1}
Feichtinger G, Novak A, Wrzaczek S.
\newblock Optimizing counter-terroristic operations in an asymmetric Lanchester
  model.
\newblock In: 15th IFAC Workshop on Control Applications of Optimization; 2012.
  .

\bibitem{Manh}
Hy MD, Vu MA, Nguyen NH, Ta AN, Bui DV.
\newblock Optimization in an asymmetric Lanchester (n,1) model.
\newblock {T}he journal of {D}efense {M}odeling and {S}imulation: Applications,
  Methodology, Technology. 2020.

\bibitem{Owen}
Owens WA.
\newblock The emergin U.S. System-of-systems.
\newblock In: U.S. Naval Institute Proceedings; 1995. .

\bibitem{Alb}
Alberts DS, Garstka JJ, Stein FP.
\newblock NETWORK CENTRIC WARFARE: Developing and Leveraging Information
  Superiority.
\newblock CCRP; 1999.

\bibitem{Smi1}
Smith CR.
\newblock Network centric warfare, command, and the nature of war.
\newblock Land warfare studies centre, Australia; 2010.

\bibitem{Tun}
Tunnell HD.
\newblock Network-Centric Warfare and the Data-Information-Knowledge-Wisdom
  Hierarchy.
\newblock Military Review. 2014.

\bibitem{Tun1}
Tunnell HD.
\newblock The U.S. Army and network-centric warfare a thematic analysis of the
  literature.
\newblock In: MILCOM 2015 - 2015 IEEE Military Communications Conference; 2015.
  p. 889--894.

\bibitem{Kim}
Kim D, Moon H, Shin H.
\newblock Some properties of nonlinear Lanchester equations with an application
  in military.
\newblock {J}ournal of {S}tatistical {C}omputation and {S}imulation. 2017.

\end{thebibliography}
\end{document}